\documentclass{article}

\usepackage{booktabs}
\usepackage{amsmath,amsthm,amsfonts,amssymb,mathtools}
\usepackage[ruled,vlined]{algorithm2e}

\usepackage{fullpage}

\newtheorem{theorem}{Theorem}

\newtheorem{lemma}[theorem]{Lemma}
\newtheorem{corollary}[theorem]{Corollary}

\newtheorem*{definition}{Definition}

\begin{document}

\title{Crowdsourced PAC Learning under Classification Noise}
\author{Shelby Heinecke and Lev Reyzin\\ \\
Department of Mathematics, Statistics, \& Computer Science\\
University of Illinois at Chicago\\
\texttt{\{sheine4,lreyzin\}@uic.edu}
}

\date{}

\maketitle

\begin{abstract}
In this paper, we analyze PAC learnability from labels produced by crowdsourcing. In our setting, unlabeled examples are drawn from a distribution and labels are crowdsourced from workers who operate under classification noise, each with their own noise parameter. We develop an end-to-end crowdsourced 
PAC learning algorithm that takes unlabeled data points as input and outputs a trained classifier. Our three-step algorithm incorporates majority voting, pure-exploration bandits, and noisy-PAC learning. We prove several guarantees on the number of tasks labeled by workers for PAC learning in this setting and show that our algorithm improves upon the baseline by reducing the total number of tasks given to workers. We demonstrate the robustness of our algorithm by exploring its application to additional realistic crowdsourcing settings.

\end{abstract}

\section{Introduction and previous work}

\subsection{Overview}

In this paper, we study the problem of learning a classifier from data labeled by a crowd of workers.  In our model, we make the assumption that each
worker has his or her own error rate, independent of the data. 
In this framework, we give a flexible three-step algorithm that achieves the PAC learning criterion. 
First, a subset of data points is chosen from $X$, and sufficiently many workers are asked to label each point, so that with high probability,
majority votes on each point are correct.  This gives a ``ground truth" set of points on which workers can be evaluated, so that in the second step, we can estimate their
individual error rates and identify good workers -- this can be done in many ways, for example by running pure-exploration bandit algorithms.   In the final step, the workers selected in the previous step are assigned
to label sufficiently many new points so that a PAC-classifier can be trained efficiently. While each part of our approach comes from known results, combining all these steps into a streamlined procedure is, to our knowledge, new.  We also illustrate 
the flexibility of our approach herein.

Instead of relying on random workers to produce labels, the goal of our approach is to quickly identify good workers
and assign the main labeling task to them.  Our algorithms work especially well when there are a few expert workers
in a large crowd, and when they are difficult to pre-screen.  Such scenarios can often occur when specialized knowledge is needed, e.g.\ in the case of using crowdsourced labels
to training a classifier to identify cat breeds, where most people presumably don't know anything about cats, but a few people 
in any large crowd will be adept at it.

\subsection{Previous work}

\paragraph{Classification noise.}
We assume that workers in the crowd are imperfect. In particular, each worker $w_i$ has an individual, hidden noise rate $0 \le \eta_i < 1/2$ so that each data point has an independent
and equal chance of being mislabeled, conditioned on the worker. We build our algorithm and analysis around this noise model but show that our analysis can be adapted to handle the case where the noise rates are conditioned on class membership. 
These noise models have been extensively studied in crowdsourcing literature \cite{CaoLTL15, Fang0CH18, KangT18, LiYZ13, WangZ15, ZhouCL14} and are usually attributed to Dawid and Skene~\cite{DawidS79}. 

In learning theory, Valiant's PAC learning model~\cite{Valiant84} was extended by Angluin and Laird~\cite{AngluinL87} to capture a simple notion of noise, which they
termed ``classification noise." In their extension, labels of samples are flipped independently with probability $0 \le \eta < 1/2$ by the noisy oracle, and the learner's runtime and sample complexity 
must also have a polynomial dependence on $\frac{1}{1-2\eta}$. For part of our work, we will adapt the results of Angluin and Laird~\cite{AngluinL87} to our noise setting. Note that our noise setting is similar in that a label of a data point is flipped independently with probability $\eta_i$ from worker $w_i$; in other words, each worker functions as a noisy oracle in our setting. Since our noise model is a generalization, we refer to our noise model as classification noise throughout this paper. 

\paragraph{Majority voting in crowdsourcing.}
Since worker skill can be unknown and varying in crowdsourcing, entities posting data points to be labeled on crowdsourcing platforms may require that each data point be labeled by multiple workers. Majority voting is the most obvious method for aggregating the labels from multiple workers.
Li~et~al.~\cite{LiYZ13} establish error rate bounds of generalized hyperplane rules of which majority voting is a special case. While they assume the same model of classification noise as our work, their analysis is limited to establishing error bounds of these hyperplane aggregation rules and not on PAC learning. Wei and Zhi-Hua~\cite{WangZ15} establish error bounds for majority voting under different assumptions, but they also do not focus on PAC learning. 
Awasthi~et~al.~\cite{AwasthiBHM17}, who focus on PAC learning from crowdsourced labels as we do, note that majority voting is not ideal because the number of worker labels needed to produce an accurate majority vote with probability $1-\delta$ scales with the size of the data set. We arrive to this same conclusion with our noise model, but we find it beneficial to still use majority voting on a small subset of the unlabeled data set to establish a ``high probability" ground-truth training set which helps to eliminate the need for queries to an expert oracle as their algorithm requires. 

\paragraph{PAC learning in crowdsourcing.}
Feng~et~al.~\cite{Fang0CH18}, in very recent work, develop PAC-style bounds for the cost complexity of learning an aggregation function that fits a crowd of workers with varying reliabilities. They focus on using PAC learning to train an aggregation function for the workers' labels; we, however, focus on using PAC learning to train a classifier that generalizes from worker labels. Concurrently, Zhang and Conitzer~\cite{ZhangC19} develop a PAC learning framework for aggregating agents' judgments in a similar setting as ours. However, they focus on recovering the target classifier exactly and employ methods similar to our baseline approach with additional assumptions. Awasthi~et~al.~\cite{AwasthiBHM17} develop PAC learning algorithms in the crowdsourcing setting that generalize from worker labels but their assumptions on the crowd differ from ours. On the one hand, they assume nothing about the workers' label distribution (this is the agnostic learning setting), but on the other hand they assume some fraction $\alpha$ of the crowd are perfect performing workers. While this assumption is reasonable in some settings (for example, if the crowd is curated), there may exist settings where this assumption would not hold since even the best performing workers are capable of making a mistake.  Thus, we instead assume that each worker has a hidden error rate $\eta_i$. Second, in the case that the fraction of perfect performing workers is less than $1/2$, their algorithm requires queries to an expert oracle.  Our algorithm, however, does not require any expert oracle queries. 
 
\paragraph{Multiarmed bandits for crowdsourcing.}
There is substantial progress in multiarmed bandit (MAB) literature regarding identifying the best arms in the vanilla MAB setting \cite{Even-DarMM06, JiangLQ17, KalyanakrishnanS10,
KangT18, MannorT04, RangiF18, ZhouCL14}. In our work, each arm will represent a worker and we build upon these previous results to train a classifier. In this work, we restrict our attention to the fixed confidence setting of best arm vanilla MAB - we want to identify the best arms with confidence $1-\delta$. 

More recently, MAB have been adapted to crowdsourcing settings \cite{CaoLTL15, KangT18, LiuL15, RangiF18, ZhangMS15, ZhouCL14}. These previous works use bandit techniques to strategically assign tasks to workers under assumptions that are realistic to crowdsourcing including limited budgets, limited worker availability, and limited worker loads. Our algorithm builds upon these works to ultimately output a trained classifier. In particular, \cite{CaoLTL15, ZhouCL14} suggest using their MAB top-$K$ arm algorithms to identify good workers, but they assume there exists a set of accurately labeled points from which to learn (ground truth set). Similarly, Liu and Liu~\cite{LiuL15} suggest using many bandit algorithms with the same limitation. Our algorithm does not require a ground truth set of data because in practice, ground truth sets may not be available or can be expensive to obtain. Although some
previous algorithms~\cite{KangT18, LiuL15} do not assume a ground truth set of points and instead estimate the correct label for points online, these algorithms describe an optimal selection policy for assigning tasks rather than training a classifier like ours. Other work \cite{RangiF18, ZhangMS15} likewise focuses on a task assignment policy rather than training a classifier. 

\section{Model and preliminaries}\label{sec:model}

Given a hypothesis class $\mathcal{C}$ of finite VC dimension $d$ and parameters $\epsilon >0$ and $\delta >0$, we 
want to PAC-learn $\mathcal{C}$ using data points with labels gathered from workers in a crowd. Let $W = \{w_i \mid i \in [1,n] \}$ denote the set of all workers, $|W| = n$. Each worker $w_i$ 
has an individual noise rate $0 \leq \eta_i \leq {1}/{2}$ 
and will correctly label any given example with probability $1 - \eta_i$.  
The noise is assumed to be persistent, so a
worker asked to label the same example a second time will deterministically produce the same label again.
In particular, for any target function $c \in \mathcal{C}$,
for any $x \in X$. The worker $w_i$ acts as follows: for all $x$,
 $\Pr[w_i(x) \neq c(x)] = \eta_i.$
 
Similar to Amazon Mechanical Turk, we define a task as a single data point that needs a label. The goal, then, is to PAC-learn $\mathcal{C}$ while minimizing the number of tasks labeled by workers; in other words, we want to minimize the number of times we query the crowd. This modeling requirement is due to the fact that in most realistic settings, workers are paid per task, and a natural
goal is to train a good classifier while expending as little as possible. Also note that this model corresponds to PAC learning from data that has been labeled by workers, each of which is a classifier operating under classification
noise~\cite{AngluinL87}. 

We will ultimately derive upper bounds on the number of tasks labeled by workers to PAC-learn $\mathcal{C}$. We now define several parameters that come into play throughout this paper and in our final bounds. 
We define 
$$\bar{\eta}_W = \frac{1}{|W|} \sum \limits_{ \{j \mid w_j \in W \} }{\eta_j}$$ 
denotes the average error rate of workers in $W$. Let $\bar{\eta}^*_{K, W}$ denote the average error rate of the best $K$ workers in $W$. As a special case, $\bar{\eta}^*_{1, W}$ denotes the error rate of the single best worker in $W$.

To obtain reliable labels, our algorithm will identify approximately good workers. Let $0 \leq \Delta \leq {1}/{2}$. We define a $\Delta$-optimal worker and $\Delta$-optimal set of $K$ workers.
\begin{definition}[$\Delta$-Optimal Worker \cite{Even-DarMM06}]
A worker $w_i \in W$ is said to be $\Delta$-optimal if $\eta_i \leq \bar{\eta}^*_{1, W} + \Delta$.
\end{definition}

\begin{definition}[$\Delta$-Optimal Set of $K$ Workers \cite{ZhouCL14}] Let $S \subseteq W$ and $|S| = K$. The set of workers $S$ is $\Delta$-optimal if $\bar{\eta}_{S} \leq  \bar{\eta}^*_{K, W}+\Delta$.
\end{definition}

\subsection{Baseline approaches}
We now describe two baseline approaches and their corresponding task complexities. For the first baseline approach, we plug $\bar{\eta}_{W}$ into the classification noise bound of Angluin (see Theorem~\ref{angluin-laird} for details)
to get an algorithm that solicits 
\begin{equation}\label{baseline}
\mathcal{O} \left(\frac{d \log\left(1/\delta\right)}{\epsilon (1-2\bar{\eta}_W)^2}\right)
\end{equation}
 labels in total, from worker pool  $W$ of $N$ workers, where workers are selected at random
to label the points. 

Another baseline approach is to obtain a large perfectly labeled set of data points with high confidence via majority voting and use the proper noiseless PAC bound, 
which requires $m = O\left(\frac{d\log(1/\delta)}{\epsilon}\right)$ examples.  
By Theorem~\ref{majvote} (which appears in Section~\ref{majorityvoting}), for each of $m$ datapoints we need
a majority vote of $O\left( \frac{ \log (m / \delta)}{ (1-2\bar{\eta}_W)^2}\right)$ workers to get a perfectly labeled set with high probability.
Combining these two bounds gives a total sample complexity again of 
$$\tilde{\mathcal{O}} \left(\frac{d \log\left(1/\delta \right)}{\epsilon (1-2\bar{\eta}_W)^2}\right),$$ 
which is actually slightly worse (with respect to polylogarithmic terms) than the bound in Equation~\ref{baseline}.

Since we want to learn to arbitrarily small errors $\epsilon$, we do not want the $d/\epsilon$ dependence to be multiplied 
by the factor of $\frac{1}{(1-2\bar{\eta}_W)^2}$, which
could be large for $\bar{\eta}_W$ close to $1/2$. This is the dependence this paper aims to avoid.


\section{Crowdsourced learning algorithm}

Our algorithm proceeds in three parts. First, we choose a small, randomly chosen set of data points to have labeled by multiple workers. This allows us to know the true labels of these data with high confidence using majority voting. Second, we use these labeled data points to identify the approximately best workers. Third, we use these workers to label additional data points from which to train a classifier.

\begin{algorithm}
\caption{Crowdsourcing PAC Algorithm (Informal)}
\textbf{Input}: $n$ workers, unlabeled data points $X$\\
\textbf{Output}: classifier $h \in \mathcal{C}$\\
  \nl take a majority vote with workers on small subset of unlabeled tasks yielding a set of accurately labeled tasks with high confidence \\
  \nl using the ground-truth data from Step 1, identify the approximate top worker(s) (e.g.\ using MAB algorithms \cite{Even-DarMM06, ZhouCL14, JiangLQ17, CaoLTL15})\\
  \nl assign tasks at random among worker(s) identified in Step 2 to perform noisy-PAC learning \cite{AngluinL87}, returning hypothesis $h \in \mathcal{C}$ consistent with the labels of the approximate top worker(s)\\
\end{algorithm}

We now proceed to analyze each part of the algorithm separately. 

\subsection{Majority voting by workers with classification noise}\label{majorityvoting}

Our algorithm begins by getting a set of points for which the labels need to be known, thereby creating a ``ground truth" set on which the workers' error rates
can be tested.  This is done by a majority vote of the labels of randomly selected workers. For this we need the following lemma, which is a 
simple consequence of the Hoeffding (it is also proved in a more general setting of Li~et~al.~\cite{LiYZ13}).

\begin{lemma}
Let $\mathcal{L}(x) = \{w_i(x) \mid w_i \in W\}$ be the labels from workers in $W$, for some $x \in X$. Suppose majority voting over the $n$ labels in $\mathcal{L}(x)$ is applied and the winning label is the final label corresponding to $x$. Then, the error of the majority vote can be upper bounded as follows:
$$
\Pr[\mathrm{MAJ}(\mathcal{L}(x)) \neq c(x)] \leq 2 e^{ {-{n (1-2\bar{\eta}_{W})^2}/{2}} }.
$$
\end{lemma}
As a consequence, we can derive the following theorem.
\begin{theorem}\label{majvote}
Let $Y \subseteq X$ and $|Y| = T$. Suppose we want to get true labels for data points in $Y$ with probability $1-\delta$ using majority voting with the crowd of workers $W$. Then for each $y \in Y$, it is sufficient to solicit 
$${\mathcal O}\left(\frac{\log(T/\delta)}{(1-2\bar{\eta}_W)^2}\right)$$
labels from the crowd.
\end{theorem}

\begin{proof}
Follows from Lemma 1 and the union bound.
\end{proof}
Since Theorem~\ref{majvote} scales poorly, it is not prudent to rely solely on majority voting for gathering a labeled data set. However, we find that using majority voting on a small enough data set can be useful because it can eliminate the assumption of a ground truth set and instead generate an ground-truth set with high probability. In this way, we also eliminate the need for expert oracle queries used in Awasthi~et~al.~\cite{AwasthiBHM17}.\footnote{
This of course relies on access to a sufficiently large crowd, and hence we assume that $N = {\tilde \Omega}\left(\frac{\log(T/\delta)}{(1-2\bar{\eta}_W)^2}\right)$, so that at this stage each worker will be assigned at most one labeling task, to get the label of each point. Additionally, notice that the number of labels in Theorem~\ref{majvote} scales as a function of the number of data points $T$ for which we want labels, as noted by Awasthi~et~al.~\cite{AwasthiBHM17}. Our bound in Theorem~\ref{majvote} is also a function of $\bar{\eta}_W$ because of our classification noise model, which differs from Awasthi~et~al.~\cite{AwasthiBHM17}.}

\subsection{Identifying top performing workers}\label{topworkers}

Using the ground-truth training labeled data set acquired from the previous section, we now identify one approximately good worker. We also examine the case where we want to identify a set of approximately good workers.

\subsubsection{Identifying one $\Delta$-optimal worker}

The naive approach to identifying a $\Delta$-optimal worker with probability $1-\delta$ is to sample each arm ${\mathcal O} \left(\frac{1}{\Delta^2}\log({n}/{\delta})\right)$ times and return the arm with the largest empirical average.
\begin{theorem}\label{deltaoptworker} 
Identifying a $\Delta$-optimal worker with probability at least $1-\delta$ can be done in ${\mathcal O} \left(\frac{n}{\Delta^2}\log({n}/{\delta})\right)$ 
arm trials.

\end{theorem}
Since each ground-truth data point can be used to test all $n$ workers, we need ${\mathcal O} \left(\frac{1}{\Delta^2}\log({n}/{\delta})\right)$ ground-truth data points to find a $\Delta$-optimal worker. If we introduce the assumption that there is at least one perfect performing worker in the crowd, then the number of arm trials to identify a $\Delta$-optimal worker decreases. 

\begin{lemma}\label{oneperfectworker}If there is at least one worker in the crowd who performs perfectly, then 
${\mathcal O}(\frac{n}{\Delta}\log({n}/{\delta}))$ samples are sufficient to identify a $\Delta$-optimal worker with probability $1-\delta$.
\end{lemma}
\begin{proof}
The probability that a worker who was observed to be perfect on $t$ examples
has error $\ge \Delta$ is bounded by $(1-\Delta)^t \le e^{-\Delta t}$.  For the union bound, we need to set 
this to $\le \delta/k$, which yields the result.
\end{proof}

\begin{corollary}\label{step2} Acquiring accurate labels for data points with probability $1-\delta$ so that a $\Delta$-optimal worker can be identified requires at most $\mathcal{\tilde{O}}
\left(\frac{\log^2(n / \delta)}{\Delta (1-2\bar{\eta}_W)^2}\right)$ 
worker labels if there is at least one perfect performing worker in the crowd and 
$\mathcal{\tilde{O}}\left(\frac{\log^2(n/\delta)}{\Delta^2(1-2\bar{\eta}_W)^2}\right)$ 
otherwise.
\end{corollary}
\begin{proof} The number of arm trials to identify a $\Delta$-optimal worker is given in Theorem~\ref{deltaoptworker} and Lemma~\ref{oneperfectworker}. We sample all arms uniformly, so ${\mathcal O} \left(\frac{1}{\Delta^2}\log({n}/{\delta})\right)$ and ${\mathcal O} \left(\frac{1}{\Delta}\log({n}/{\delta})\right)$ accurately labeled points are needed in order to compute the reward for each arm trial, respectively. We acquire an accurately labeled point with high confidence as in Theorem~\ref{majvote}, where we set 
$T = {\mathcal O} \left(\frac{1}{\Delta^2}\log({n}/{\delta})\right)$ and $T = {\mathcal O} \left(\frac{1}{\Delta}\log({n}/{\delta})\right)$. We then multiply by $T$ to get the total number of worker labels needed to acquire an ground-truth set of size $T$. 
\end{proof}

The problem of identifying the best workers can also be solved with
sophisticated methods that 
employ pure-exploration stochastic multi-armed bandit algorithms~\cite{CaoLTL15, ZhouCL14}; for example, \texttt{OptMAI} (see Theorem~\ref{topk}) improves the dependence on $n$
in logarithm, even in the case of finding the approximately-best worker.

In our crowdsourcing setting, each worker is an arm in the bandit setting with mean reward $1-\eta_i$. When we select a worker/arm, the reward is $1$ if the worker's label is correct and $0$ otherwise. In order to compute rewards, many bandit algorithms require a ground-truth set of points \cite{CaoLTL15, Even-DarMM06, JiangLQ17, ZhouCL14}. Instead, we use the set we gathered from the majority voting step as a proxy for a ground-truth set. Thus, we are able to make use of many MAB algorithms, but for now we focus on vanilla MAB.

\subsubsection{Identifying the top $K$ workers}
Let $K \leq n$. The following sample complexity bound on identifying a set of the approximate top $K$ workers is known. 
\begin{theorem}[Zhou~et~al.~\cite{ZhouCL14}]\label{topk}
For $K \leq \frac{n}{2}$, \texttt{OptMAI}($n$, $K$, $q$) 
identifies a $\Delta$-optimal set of $K$ arms with probability $1-\delta$ using 
\begin{equation}\label{eq:optami}
q = {\mathcal O}\left(\frac{n}{\Delta^2}\left(1+\frac{\log({1}/{\delta})}{K}\right)\right)
\end{equation} 
arm trials.\footnote{For $K \geq {n}/{2}$, 
\texttt{OptMAI}($n$,$K$, $q$) identifies a $\Delta$-optimal set of $K$ arms with probability $1-\delta$ using $$q = {\mathcal O}\left(\left(\frac{(n-K)n}{K\Delta^2}\right)\left(\frac{(n-K)}{K} + \frac{\log({1}/{\delta})}{K}\right)\right)$$ arm trials.}
\end{theorem}

An upper bound on the number of trials per arm is given, as well. 
\begin{theorem}[Zhou~et~al.~\cite{ZhouCL14}]\label{perarm}
In \texttt{OptMAI}($n$, $K$, $q$), each arm is sampled at most
\begin{equation}\label{eq:perarmOPTAMI}
s= \mathcal{O} \left(\frac{q}{n^{.3}}\right)
\end{equation}
times, where $q$ is set according to Equation~\ref{eq:optami}.
\end{theorem}

Now, we will use \texttt{OptMAI} from above to efficiently learn a set of the approximate top $K$ workers in our crowdsourcing model.
A worker labeling a datapoint will function as an arm pull.  Hence, a number of correctly labeled 
datapoints as in Equation~\ref{eq:perarmOPTAMI} will be sufficient to implement this strategy.

\begin{corollary}\label{karms} Acquiring accurate labels for $s$ data points (as per Equation~\ref{eq:perarmOPTAMI}) with probability $1-\delta$ so that a $\Delta$-optimal set of $K$ workers can be identified requires at most 
\[
\mathcal{\tilde{O}}
\left(\frac{n^{.7}\log({1}/{\delta})\left(1+\frac{\log\left({1}/{\delta}\right)}{K}\right)}{\Delta^2 ({1}-2\bar{\eta}_W)^2}\right)
\]
total tasks assigned to workers to label points.
\end{corollary}

\begin{proof}
The number of arm trials to identify a $\Delta$-optimal set of $K$ workers is given by $q$ in Theorem~\ref{topk}. Each arm is sampled at most $\mathcal{O} \left(\frac{q}{n^{.3}}\right)$ times (Theorem~\ref{perarm}), thus we need this many accurately labeled points in order to compute the reward for each arm trial. We acquire an accurately labeled point with high confidence as in 
Theorem~\ref{majvote}, where we set $T = \mathcal{O} \left(\frac{q}{n^{.3}}\right)$. We then multiply by $T$ to get the total number of worker labels needed to acquire a ground-truth set of size $T$. 
\end{proof}

It is also clear that for various extensions and variants of our problem, we can also use more 
sophisticated bandit algorithms.  For example, if different sets of workers are available during different
rounds, we can use sleeping bandits~\cite{KleinbergNS10}, etc.  The variety of known bandit algorithms
working under various assumptions further illustrates the flexibility of our modular approach.

\subsection{PAC learning under label noise}\label{PAC}
Now that the algorithm has identified good workers, we use those workers to label more tasks needed to PAC learn the concept class $\mathcal{C}$. In this step, each task consists of labeling a distinct data point; in other words, each data point is labeled only once by one of the good workers we identified in the previous step. To perform the PAC learning we use the algorithm from \cite{AngluinL87} in which the learner queries a noisy oracle sufficiently many times and returns the hypothesis $h \in \mathcal{C}$ that has the minimal number of disagreements with the results from the noisy oracle. We assume that finding this hypothesis can be done efficiently. 

We first recall the main result from \cite{AngluinL87}, which states the number of queries that must be made to the noisy oracle in order to PAC learn $\mathcal{C}$:
\begin{theorem}[Angluin and Laird~\cite{AngluinL87}]\label{angluin-laird}
If a learning algorithm that is given at least $$O \left( \frac{d \log\left({1}/{\delta}\right)}{\epsilon (1-2\eta)^2}\right)$$ samples from a noisy oracle with 
error parameter $0 \leq \eta < {1}/{2}$ can produce a hypothesis $h_S$ that 
minimizes disagreements with the noisy sample, 
then $h_S$ satisfies the PAC criterion for the class $\mathcal{C}$, i.e. for any $\epsilon, \delta >0$ and any distribution $\mathcal{D}$ on $X$,
\[
\Pr_{S \sim D^m} (d(h_S, h^*) \geq \epsilon) \leq \delta,
\]
where $d(h_S,h^*)$ denotes the rate of disagreement between $h_S$ and the target concept $h^*$.
\end{theorem}
We adapt Theorem~\ref{angluin-laird} to the two types of approximately good workers identified in the previous section so that either one $\Delta$-optimal worker functions as the noisy oracle or the $\Delta$-optimal set of $K$ workers sampled i.i.d. function as the noisy oracle. In the former case, the oracle noise rate $\eta$ becomes $\bar{\eta}^*_{1,W} + \Delta$, so by Theorem~\ref{angluin-laird} we assign to the $\Delta$-optimal worker at most 
\begin{equation}\label{bound1}
{\mathcal O}\left(\frac{d\log({1}/{\delta})}{\epsilon (1-2(\bar{\eta}^*_{1,W}+\Delta))^2}\right)
\end{equation}
additional points to label. 
In the latter case, the oracle noise rate becomes $\bar{\eta}^*_{K,W}+\Delta$, so we assign the $\Delta$-optimal set of $K$ workers
\begin{equation}\label{bound2}
{\mathcal O}\left(\frac{d \log({1}/{\delta})}{\epsilon (1-2(\bar{\eta}^*_{K,W}+\Delta))^2}\right)
\end{equation}
points to label.

\subsection{Total task complexity}\label{finalbounds}
We combine the task bounds established for majority voting (Section~\ref{majorityvoting}), identifying good workers (Section~\ref{topworkers}), and PAC-learning with good workers (Section~\ref{PAC}) to derive the total task complexity of our algorithm. In order for our algorithm to be a PAC learning algorithm for $\mathcal{C}$, we recall
that for any $\epsilon, \delta >0$ and any distribution $\mathcal{D}$ on $X$ from which a sample $S$ is drawn, the hypothesis $h_S \in C$ returned by the algorithm must satisfy
\[
\Pr_{S \sim \mathcal{D}^m}[\Pr_{x \sim \mathcal{D}}(h_S(x) \neq c(x)) \geq \epsilon] \leq \delta
\]
with a sample complexity that is $\mathrm{poly}({1}/{\epsilon}, {1}/{\delta}, |x|, \mathrm{size}(c))$. To satisfy this PAC criterion, we set the failure rate for each part of our algorithm to be at most $\delta/3$ (so that the total failure rate is bounded by $\delta$). We then add the three task bounds. Notice that the bounds in Equations~\ref{bound1} and~\ref{bound2} adapted from Theorem \ref{angluin-laird} and the arm trial bounds from Section \ref{topworkers} are a function of $\Delta$. We parameterize $\Delta$ as a function of either best worker's error rate or the average error rate of the best set of $K$ workers. For Theorems~\ref{result1} and~\ref{result2}, which follow, we set
$$
\Delta = \frac{{1}/{2} - \bar{\eta}^*_{1,W}}{2}.
$$
The following theorem gives an upper bound on the number of tasks required by our algorithm in order to PAC-learn $\mathcal{C}$.
\begin{theorem}\label{result1} Let $\epsilon, \delta > 0$. Suppose that in Step 2 of the algorithm, we identify one approximately good worker. Then
$$\mathcal{\tilde{O}}\left(\frac{\log^2({n}/{\delta})}{({1}-2\bar{\eta}^*_{1,W})^2({1}-2\bar{\eta}_W)^2} + \frac{\left(n+\frac{d}{\epsilon}\right)\log({1}/{\delta})}{({1}-2\bar{\eta}^*_{1,W})^2}\right)$$ 
tasks can be labeled by workers in order to efficiently PAC learn $\mathcal{C}$.
\end{theorem}

\begin{proof} We sum the task complexity from each step in the algorithm. We first sample the crowd $\mathcal{\tilde{O}}\left(\frac{\log^2(n/\delta)}{\Delta^2(1-2\bar{\eta}_W)^2}\right)$ times (Corollary~\ref{step2}) in order to gather a ground truth set with probability $1-\delta$. Using the ground truth set as the training set, we sample the crowd ${\mathcal O} \left(\frac{n}{\Delta^2}\log({n}/{\delta})\right)$ times (Theorem~\ref{deltaoptworker}) in order to identify an approximately good worker. We then use the approximately good worker to label ${\mathcal O}\left(\frac{d \log({1}/{\delta})}{\epsilon (1-2(\bar{\eta}^*_{1,W}+\Delta))^2}\right)$ points (bound in Equation~\ref{bound1}). Summing these components gives
\[
{\mathcal{\tilde{O}}} \left(\frac{\log^2(n/\delta)}{\Delta^2(1-2\bar{\eta}_W)^2} + \frac{n}{\Delta^2}\log({n}/{\delta}) + 
\frac{d \log({1}/{\delta})}{\epsilon (1-2(\bar{\eta}^*_{1,W}+\Delta))^2}\right).
\]
Setting $\Delta = \frac{{1}/{2} - \bar{\eta}^*_{1,W}}{2}$ and simplifying yields the task complexity.
 \end{proof}

Recall that from Lemma~\ref{oneperfectworker}, if we assume there is one perfect worker in the crowd, the task complexity improves. We see the improvement in the overall task complexity below. In this case, since $\bar{\eta}^*_{1,W} = 0$, we set $\Delta = \frac{1}{4}$.
\begin{theorem}\label{result2}  Let $\epsilon, \delta > 0$. Suppose that in Step 2 of the algorithm, we identify one approximately good worker and we assume there exists at least one perfect performing worker in the crowd. Then  
$$\mathcal{\tilde{O}}\left(\frac{\log^2({n}/{\delta})}{({1}-2\bar{\eta}_W)^2} + \left(n+ \frac{d}{\epsilon}\right)\log({1}/{\delta})\right)$$ 
tasks can be labeled by workers in order to efficiently PAC learn $\mathcal{C}$.
\end{theorem}

\begin{proof} We sum the task complexity from each step in the algorithm. We first sample the crowd $\mathcal{\tilde{O}}\left(\frac{\log^2(n/\delta)}{\Delta(1-2\bar{\eta}_W)^2}\right)$ times (Corollary~\ref{step2}) in order to gather a ground truth set with probability $1-\delta$. Using the ground truth set as the training set, we sample the crowd ${\mathcal O} \left(\frac{n}{\Delta}\log({n}/{\delta})\right)$ times (Lemma~\ref{oneperfectworker}) in order to identify an approximately good worker. We then use this approximately good worker to label ${\mathcal O}\left(\frac{d \log({1}/{\delta})}{\epsilon (1-2(\bar{\eta}^*_{1,W}+\Delta))^2}\right)$ points (Equation~\ref{bound1}). Since there is a perfect worker in the crowd, $\bar{\eta}^*_{1,W} = 0$. Summing these components gives
\[
{\mathcal{\tilde{O}}} \left(\frac{\log^2(n/\delta)}{\Delta(1-2\bar{\eta}_W)^2} + \frac{n}{\Delta}\log({n}/{\delta}) + 
\frac{d \log({1}/{\delta})}{\epsilon (1-2\Delta)^2}\right).
\]
Setting $\Delta = {1}/{4}$ and simplifying yields the task complexity.
 \end{proof}

As discussed in Section~\ref{topworkers}, an alternative to identifying one approximately good worker is to identify $K$ approximately good workers to limit the burden of tasks for workers. The maximum number of tasks a single worker must complete is referred to as the load \cite{AwasthiBHM17}. In the case of one approximately good worker, that worker must label all the tasks prescribed by the bound in Equation~\ref{bound1}. In the case of $K$ approximately good workers, the workers can evenly split the tasks prescribed by the bound in Equation~\ref{bound2}, reducing the load. If load is a priority in a particular crowdsourcing setting, then we have the following task upper bound for our algorithm. To derive this bound, we set 
$$\Delta = \frac{{1}/{2} - \bar{\eta}^*_{K,W}}{2}.$$

\begin{theorem}\label{result3} Let $\epsilon, \delta > 0$. Let $K$ denote the number of workers identified in Step 2 of the algorithm and assume $K \leq \frac{n}{2}$.
Then
$$\mathcal{\tilde{O}}\left(\frac{n^{.7} \log({1}/{\delta})(1+\frac{1}{K}\log({1}/{\delta}))}{({1} - 2\bar{\eta}^*_{K,W})^2({1} - 2\bar{\eta}_W)^2} + \frac{(\frac{n}{K} + \frac{d}{\epsilon})\log({1}/{\delta})+n}{ ({1} -2 \bar{\eta}^*_{K,W})^2}\right)
 \subset
\mathcal{\tilde{O}}\left(\frac{n \log^2({1}/{\delta})}{({1} - 2\bar{\eta}^*_{K,W})^2({1}- 2\bar{\eta}_W)^2} + 
\frac{d \log({1}/{\delta})}{ \epsilon ({1} - 2\bar{\eta}^*_{K,W})^2}\right)
$$
tasks can be labeled by workers in order to efficiently PAC learn $\mathcal{C}$.
\end{theorem}

\begin{proof} Again, we sum the task complexity from each step in the algorithm. We first sample the crowd $\mathcal{\tilde{O}}
\left(\frac{n^{.7}\log({1}/{\delta})(1+\frac{\log({1}/{\delta})}{K})}{\Delta^2 (1- 2 \bar{\eta}_W)^2}\right)$ times (Corollary~\ref{karms}) in order to gather a ground truth set with probability $1-\delta$. Using the ground truth set as the training set, we sample the crowd $q$ times in order to identify a set of $K$ approximately good workers (Theorem~\ref{topk}). We then use these approximately good workers to label ${\mathcal O}\left(\frac{d\log({1}/{\delta})}{\epsilon (1-2(\bar{\eta}^*_{K,W}+\Delta))^2}\right)$ points (bound in Equation~\ref{bound2}). Summing these components gives
$$
\mathcal{\tilde{O}}
\left(\frac{n^{.7}\log({1}/{\delta})(1+\frac{\log({1}/{\delta})}{K})}{\Delta^2 ({1}- {2}\bar{\eta}_W)^2}\right) + {\mathcal O}\left(\frac{n}{\Delta^2}\left(1+\frac{\log({1}/{\delta})}{K}\right)\right)  + {\mathcal O}\left(\frac{d\log({1}/{\delta})}{\epsilon (1-2(\bar{\eta}^*_{K,W}+\Delta))^2}\right).
$$
Setting $\Delta = \frac{{1}/{2} - \bar{\eta}^*_{K,W}}{2}$ and simplifying yields the task complexity.
 \end{proof}

\subsection{Comparison to baseline and to other work}
In the bounds established above, the term $\frac{1}{({1}- {2}\bar{\eta}_W)^2}$ is not multiplied the $d/\epsilon$ term, which is the improvement over
the baseline described in Section~\ref{sec:model}. In particular, in Theorem~\ref{result1}, the $d/\epsilon$ term is multiplied by a factor of
$$\frac{1}{({1} -{2}\bar{\eta}^*_{1,W})^2}$$ which is function of the error rate of the best worker $\bar{\eta}^*_{1,W}$ in $W$ instead of the average of all workers in $W$, $\bar{\eta}_W$, as in the baseline. Similarly, in in Theorem~\ref{result3}, the $d/\epsilon$ term is multiplied by $$\frac{1}{ ({1}-{2} \bar{\eta}^*_{K,W})^2}$$ which is function of the error rate of the best $K$ workers in $W$, $\bar{\eta}^*_{K,W}$, instead of $\bar{\eta}_W$. Theorem~\ref{result2} shows further improvement from the baseline as the $d/\epsilon$ is multiplied only by a factor of $\log({1}/{\delta})$. Note that in all three Theorems, the task complexity can get arbitrarily bad as any of the crowd parameters approaches random guessing, i.e. as $\bar{\eta}^*_{1,W}$, $\bar{\eta}_W$, or $\bar{\eta}^*_{K,W}$ approach ${1}/{2}$. 

Unlike the baseline, there are additional terms in each of the bounds above that are not multiplied by $d/\epsilon$. While these terms indeed add to the task complexity, as $\epsilon$ becomes arbitrarily small they become negligible, thus, the term multiplied by $d/\epsilon$ is most important. 

We now discuss how our results compare to the work of Awasthi~et~al.~\cite{AwasthiBHM17}. Recall that Awasthi~et~al.~\cite{AwasthiBHM17} assume that a fraction $\alpha$ of workers are perfect performers with no assumptions on the rest of the crowd. Like Awasthi~et~al.~\cite{AwasthiBHM17}, our algorithm is a PAC learning algorithm but ours does not rely on or require an assumption of perfect workers in the crowd. Instead, our algorithm assumes everyone has an individual noise rate. When we do consider perfect workers, we find that even just one perfect worker in the crowd improves our task complexity bound. When the fraction of perfect workers is below ${1}/{2}$, the algorithm in Awasthi~et~al.~\cite{AwasthiBHM17} requires ``golden queries", queries to an expert oracle. Note that none of our PAC bounds are dependent on access to an expert oracle. 

\section{Variants and extensions}
We now demonstrate a few ways in which our model and algorithm can be easily adapted to fit different crowdsourcing settings. 

\subsection{Asymmetric classification noise}
In some settings, workers may perform differently depending on the true label of the data point. This assymetric noise model is attributed to Dawid and Skene \cite{DawidS79}. For simplicity, we assume the binary classification setting with labels $\{-1, +1\}$.
 For worker $w_i \in W$, let $\eta_i^+$ and $\eta_i^-$ denote the error rates of positive and negative instances; that is, for each $x \in X$, $$\eta_i^+ = \Pr[w_i(x) \neq c(x) \mid c(x) = 1]$$ and $$\eta_i^+ = \Pr[w_i(x) \neq c(x) \mid c(x) = -1].$$ Let $\bar{\eta}^{+}_W$ and $\bar{\eta}^{-}_W$ denote the average one-sided error rates among all workers in $W$.
Also let $$\hat{\eta}_i = \eta_i^+\Pr[c(x)=1]+\eta_i^-\Pr[c(x)=-1]$$ and let $i^* = \mathrm{argmin}_i \hat{\eta}_i$.

We now show that our algorithm can be easily adapted to the setting of asymmetric classification noise. 
We first derive an analogue of Theorem~\ref{majvote} which is also a result of Hoeffding and union bounds.
\begin{theorem}\label{majvote2}
Let $Y \subseteq X$ where $|Y| = T$. Suppose we want to get true labels for data points in $Y$ with probability $1-\delta$ using majority voting using the crowd of workers $W$ under asymmetric classification noise. Then for each data point $y \in Y$, it is sufficient to solicit 
$${\mathcal O}\left(\frac{\log({T}/{\delta})}{(1-2 \max (\bar{\eta}^{+}_W, \bar{\eta}^{-}_W))^2 }\right)$$
labels. 
\end{theorem}

The only added bound we need is an asymmetric-noise analogue for the Angluin and Laird~\cite{AngluinL87} bound from Theorem~\ref{angluin-laird}.
\begin{corollary}[to Theorem~\ref{angluin-laird}]\label{cor:AL}
If a learning algorithm that is given at least $$O \left( \frac{d \log\left({1}/{\delta}\right)}{\epsilon (1-2\max(\eta^+,\eta^-))^2}\right)$$ samples labeled under the Dawid-Skene noise model~\cite{DawidS79}
with parameters $\eta^+$ and $\eta^-$ can produce a hypothesis $h$ that 
minimizes disagreements with the noisy sample, 
then $h$ satisfies the PAC criterion for the class $\mathcal{C}$, i.e. for any $\epsilon, \delta >0$ and any distribution $\mathcal{D}$ on $X$,
\[
\Pr (d(h, h^*) \geq \epsilon) \leq \delta.
\]
where $d(h,h^*)$ denotes the rate of disagreement between $h$ and the target concept $h^*$.
\end{corollary}
\begin{proof}
The simplest proof of this is a reduction to the uniform noise case, as suggested by Blum and Kalai~\cite{BlumK98} for reducing from
one-sided noise to two-sided noise.  Without loss of generality, assume $\eta^+ > \eta^-$ (otherwise, we will flip the other label).  We will flip each negative label with probability $p$.
Hence, the new noise rates are $\eta'^{+} = \eta^+ - \eta^+p$ and $\eta'^{-} = \eta^- + (1-\eta^-)p$.  Making $\eta'^- = \eta'^+$ and 
solving for $p$ yields $p = \frac{\eta^+ - \eta^-}{1+\eta^+ - \eta^-}$.  The new symmetric noise rate is 
now $\eta'^+ = \eta'^- \le \max(\eta^+, \eta^-)$, so we can apply the bound from Theorem~\ref{angluin-laird} to finish the proof.
\end{proof}

We can now proceed as in Sections~\ref{topworkers}, \ref{PAC}, and~\ref{finalbounds} to derive upper bounds on the task complexity of our algorithm adapted to this new setting.

\begin{theorem}\label{newresult1} Let $\epsilon, \delta > 0$. Suppose we identify one approximately good worker in Step 2 of the algorithm per label. Then
$$\mathcal{\tilde{O}}\left(\frac{\log^2({n}/{\delta})}{(1-2\max (\bar{\eta}^{+}_W, \bar{\eta}^{-}_W))^2({1}-{2}\hat{\eta}_{i^*})^2} + \left(n+\frac{d}{\epsilon}\right)\frac{\log({1}/{\delta})}{({1}-{2}\max({\eta}^{+}_{i^*},{\eta}^{-}_{i^*}))^2}\right)$$ 
tasks can be labeled by workers in order to efficiently PAC learn $\mathcal{C}$.
\end{theorem}

\begin{proof} We sum the task complexity from each step in the algorithm. As before, we first sample the crowd  $\mathcal{\tilde{O}}\left(\frac{\log^2(n/\delta)}{\Delta(1-2 \max ({\eta}^{+}_W, {\eta}^{-}_W))^2}\right)$ times (Theorem~\ref{deltaoptworker} and Theorem~\ref{majvote2}) in order to gather a ground truth set with probability $1-\delta$. Using the ground truth set as the training set, we sample the crowd ${\mathcal O} \left(\frac{n}{\Delta^2}\log({n}/{\delta})\right)$ times (Theorem~\ref{deltaoptworker}) in order to identify an approximately good worker. We then use the approximately good worker to label $\mathcal{O} \left( \frac{d \log\left({1}/{\delta}\right)}{\epsilon (1-2\max({\eta}^{+}_{i^*},{\eta}^{-}_{i^*}))^2}\right)$ points (Corollary~\ref{cor:AL}). Summing these components gives
$$
\mathcal{\tilde{O}}\left(\frac{\log^2(n/\delta)}{\Delta(1-2\max (\bar{\eta}^{+}_W, \bar{\eta}^{-}_W)^2)}\right) + {\mathcal O} \left(\frac{n}{\Delta^2}\log({n}/{\delta})\right) + \mathcal{O} \left( \frac{d \log\left({1}/{\delta}\right)}{\epsilon (1-2\max({\eta}^{+}_{i^*},{\eta}^{-}_{i^*}))^2}\right).
$$
Setting $\Delta = \frac{{1}/{2} - \hat{\eta}_{i^*}}{2}$ and simplifying yields the task complexity.
 \end{proof}

The baseline approach in this setting would be to substitute the average one-sided error rates of workers into the bound 
from Corollary~\ref{cor:AL}, yielding an upper bound of
$$\mathcal{O} \left(\frac{d \log\left(1/\delta\right)}{\epsilon (1-2 \max (\bar{\eta}^{+}_W, \bar{\eta}^{-}_W))^2}\right).$$
With only slight adjustments to our algorithm and analysis, the bound we derive in Theorem~\ref{newresult1} for this asymmetric noise setting is still an improvement on the baseline since the ${d}/{\epsilon}$ term is multiplied by a term that is a function of ${\bar{\eta}}^{+}_{i^*}$ and ${\bar{\eta}}^{-}_{i^*}$ instead of ${\bar{\eta}}^+_{W}$ 
and ${\bar{\eta}}^-_{W}$.

\subsection{Per-worker task limits}
In practice, the load per worker may be limited. In Section 3.4 we discuss how identifying the top $K$ workers instead of one good worker in Step 2 of the algorithm reduces the load on workers. In this extension, we consider a different setting of limited worker loads where each worker can complete no more than $B$ tasks. This setting is useful because in reality, workers will have time and energy limitations that will bound the number of tasks they can realistically complete. Suppose that $B > 0$ denotes the task limit for each worker. We assume that $B$ is greater than the number of trials per arm required to identify top workers in Step 2 so that each worker has the capacity to help with labeling additional tasks, to some extent, in Step 3. 
We assume that in this setting, for Step 2 of the algorithm, we identify a $\Delta$-optimal set of $K$ workers instead of one $\Delta$-optimal worker. We take this approach because we can use the limited worker load constraint and our knowledge of the number of tasks to be completed in Step 3 of the algorithm to determine how many workers $K$ to identify in Step 2. We first determine the capacity remaining per worker after running the top-$K$ MAB algorithm. We recall that $$q = {\mathcal O}\left(\frac{n}{\Delta^2}\left(1+\frac{\log({1}/{\delta})}{K}\right)\right)$$ from Theorem 3.6.

\begin{lemma}\label{remaining} After implementing \texttt{OptMAI($n$,$K$,$q$)} to identify the top $K \leq {n}/{2}$ workers, the number of tasks remaining per worker is at least 
$$B - \frac{4n^{.7}}{({1}/{2} - \bar{\eta}^*_{K,W})^2}\left(1+\frac{\log({1}/{\delta})}{K}\right).$$ 
\end{lemma}
\begin{proof} From Theorem~\ref{perarm}, each arm is pulled at most ${\mathcal O}({q}/{n^{.3}})$ times. We subtract this from the task limit $B$
and set $$\Delta = \frac{{1}/{2} - \bar{\eta}^*_{K,W}}{2},$$ and simplify.
\end{proof}
Recall that the bound in Equation~\ref{bound2}, derived from Theorem~\ref{angluin-laird}, is the number of data points that need to labeled by the selected $K$ workers in Step 2 to complete the PAC learning algorithm. Dividing the bound from Equation~\ref{bound2} by the capacity remaining per worker yields the number of top workers, $K$, that need to be identified in Step 2.

\begin{lemma}\label{K} In Step 2 of the algorithm, $$K = {\mathcal O}\left(\left(\frac{d}{\epsilon}+n^{.7}\right)\frac{\log({1}/{\delta})}{B(1 - 2\bar{\eta}^*_{K,W} )^2 - n^{.7}}\right)$$
workers will be identified.
\end{lemma}
\begin{proof}The number of workers $K$ is the number of data points that need to be labeled, as prescribed by the bound in Equation~\ref{bound2}, divided by the tasks remaining per worker, established in Lemma~\ref{remaining}. After setting $\Delta = \frac{{1}/{2} - \bar{\eta}^*_{K,W}}{2}$, $$K=  \frac{\frac{2d}{\epsilon ({1} - 2\bar{\eta}^*_{K,W})^2}\log({1}/{\delta})}{B - \frac{2n^{.7}}{({1} - 2\bar{\eta}^*_{K, W})^2}\left(1+\frac{\log({1}/{\delta})}{K}\right)}.$$ 
The theorem follows from solving for $K$.
\end{proof}
Theorem~\ref{result3} can now be extended to upper bound the number of tasks labeled by workers in this new setting by simply letting $K$ be defined as in Lemma~\ref{K}.

\subsection{Agnostic PAC}
It is possible that our symmetric or asymmetric classification noise model does not model the behavior of all workers. 
For instance, there may be workers who behave maliciously or workers with error rates $\eta_i > {1}/{2}$. 
On one end of the spectrum, each worker may have a fixed error rate. 
On the other end of the spectrum, there may be no assumptions on worker behavior at all, and this case is referred to as the agnostic setting.

PAC learning in the agnostic setting is usually computationally hard~\cite{FeldmanGRW12}.
Hence, Awasthi~et~al.~\cite{AwasthiBHM17} assume an $\alpha$ fraction of workers are perfect performers and places no assumptions on the behavior of the remaining $1-\alpha$ workers. We would like to begin bridging the two ends of the spectrum in a similar way to account for workers that cannot be modeled by classification noise.  
Here, we begin to do this by showing a simple extension of the result of Awasthi~et~al.~\cite{AwasthiBHM17}.

We let $\alpha$ denote the fraction of workers that can be modeled by persistent classification noise with $\eta_i \leq {1}/{2}$. As in the work of Awasthi~et~al.~\cite{AwasthiBHM17}, there are no assumptions on the behavior of the remaining $1-\alpha$ fraction of workers. PAC learning can be achieved in this setting by adapting the proof of Theorem 4.3 from
Awasthi~et~al.~\cite{AwasthiBHM17}; in particular, in our proposed setting, the probabilistic guarantees of their Lemma~4.6 still hold. Thus, their algorithm still the extended setting we proposed, as we state in the following corollary.

\begin{corollary}[to Theorem 4.3 of Awasthi~et~al.~\cite{AwasthiBHM17}] Let $W$ be the subset of workers that can be modeled by persistent classification noise with average noise rate $\bar{\eta}_W$,
and let $\alpha$ denote the fraction all workers that are in $W$.
Then when $$\alpha(1-\bar{\eta}_W ) \geq {1}/{2},$$  
a concept class $\mathcal{C}$ can be efficiently PAC learned from the crowd given the ability to efficiently find an ERM over $\mathcal{C}$.
\end{corollary}

\section*{Acknowledgements} 
This work was supported in part by NSF grant CCF-1848966.

\bibliographystyle{plain}
\bibliography{paper}

\end{document}